\documentclass[11pt]{article}
\usepackage{microtype}
\usepackage{graphicx,wrapfig}
\usepackage{subcaption}
\usepackage{amsmath,amsbsy,amsfonts,amssymb,amsthm,bm}
\usepackage{algorithm,algorithmic,mathtools}
\usepackage{color,cases,multirow}
\usepackage[margin=1in]{geometry} 
\usepackage[colorlinks=True, citecolor=blue]{hyperref}
\usepackage{float}

\usepackage{authblk}

\newtheorem{theorem}{Theorem}
\newtheorem{remark}{Remark}

\newtheorem{lemma}{Lemma}

\theoremstyle{definition}
\newtheorem{definition}{Definition}

\DeclareMathOperator{\argmin}{argmin}

\DeclareMathOperator{\tr}{Tr}

\newcommand{\RR}{\mathbb{R}}

\newcommand{\EE}{\mathbb{E}}

\newcommand{\ZZ}{\mathbb{Z}}

\newcommand{\bx}{\mathbf{x}}

\newcommand{\bp}{\mathbf{p}}

\newcommand{\by}{\bm{y}}

\newcommand{\bi}{\begin{enumerate}}
\newcommand{\ei}{\end{enumerate}}

\title{Beyond the Quadratic Approximation: the Multiscale Structure of Neural Network Loss Landscapes}

\author[1]{Chao Ma}
\author[2]{Daniel Kunin}
\author[3]{Lei Wu}
\author[1]{Lexing Ying}

\affil[1]{Department of Mathematics, Stanford University}
\affil[2]{Institute for Computational and Mathematical Engineering, Stanford University}
\affil[3]{School of Mathematical Sciences, Peking University}

\date{\today}

\begin{document}

\maketitle 

\begin{abstract}
A quadratic approximation of neural network loss landscapes has been extensively used to study the optimization process of these networks.
Though, it usually holds in a very small neighborhood of the minimum, it cannot explain many phenomena observed during the optimization process. In this work,
we study the structure of neural network loss functions and its implication on optimization in a
region beyond the reach of a good quadratic approximation.  Numerically, we observe that neural
network loss functions possesses a {\it multiscale structure}, manifested in two ways: (1) in a
neighborhood of minima, the loss mixes a continuum of scales and grows subquadratically, and (2) in
a larger region, the loss shows several separate scales clearly. Using the subquadratic growth, we
are able to explain the Edge of Stability phenomenon~\cite{cohen2021gradient} observed for the gradient
descent (GD) method. Using the separate scales, we explain the working mechanism of learning rate
decay by simple examples. Finally, we study the origin of the multiscale structure and propose that
the non-convexity of the models and the non-uniformity of training data is one of the causes. By constructing a two-layer neural network
problem we show that training data with different magnitudes give rise to different scales of the
loss function, producing subquadratic growth and multiple separate scales.
\end{abstract}


\section{Introduction}
Despite the well-known nonconvexity of neural network's loss functions, utilizing local quadratic approximations around (global or local) minimum has been a fruitful approach to study the optimization behaviors for deep neural networks. For example, global convergence of the gradient descent (GD) method can be established for many neural network models in the so-called ``lazy training'' (or neural tangent kernel (NTK)) regime, where the training trajectory stays within a region with good quadratic approximation~\cite{Jacot2018Neural,du2018gradient,ma2019comparative}. Additionally, quadratic approximations around global minima can be used to explain the preference of stochastic gradient descent (SGD) for flat minima~\cite{wu2018sgd,dai2018towards} and certain properties of the limiting dynamics~\cite{kunin2021limiting}.


However, regardless of the theoretical simplicity brought by local quadratic approximation, empirical study has shown that the effect of higher-order information is far from negligible in most stages of the training and for most reasonable choices of hyperparameters. One reflection of the important role played by non-quadratic terms is the Edge of Stability (EoS) phenomenon~\cite{cohen2021gradient}. The EoS phenomenon shows that GD can always find and stay in the sharpest region that it can be stable. As a sharp contrast, in a quadratic landscape GD either converges or blows up in most cases. 
Another relevant observation is the effect of learning rate decay. After a learning rate decay, the trajectory changes its moving direction and converges to a different solution. This cannot be explained by a quadratic approximation of the loss. 
Including non-quadratic and nonconvex information of the loss function into the study of neural network optimization is necessary to acquire more realistic understanding of the behavior of the optimizers, even locally around the minimum.

In this work, motivated by the two problems mentioned above, we study the behavior of optimizers on neural network loss functions beyond local quadratic approximation. We start from empirical observations by visualizing the loss landscape around the training trajectory. Then, from the observations we extract relevant simplified models and theoretically study the optimization dynamics on these simplified problems. Specifically, we obtain two typical observations: (1) around global minimum, the loss grows slower than a quadratic function, which we name the \textbf{subquadratic growth}; (2) in a larger region, the loss function shows distinct scales, which we name the \textbf{separate scales structure}.
For the former, we propose to take a subquadratic function, a function that gets flatter when the parameters get farther from the minimum, as local approximation of the loss function. We study the behavior of GD minimizing this function and explain the mechanism behind the EoS phenomenon. We also consider a minima manifold with subquadratic landscape in the directions that are orthogonal to the manifold, and study the motion of the GD along the manifold driven by flatness after reaching EoS. For the latter observation (the separate scales structure), we consider a landscape with several valleys in different scales and explain the working mechanism of learning rate decay. 
Our theoretical studies, though not directly conducted on real neural network loss functions, help us build insights on what is happening during the training process. 

In addition to characterizing the optimization behavior, we are also concerned with the origin of the observed properties of neural network loss functions. We understand both subquadratic growth and separate scales as manifestation of multiscale structures---a continuum of scales for subquadratic growth and finite scales for separate scales structure. 
By a construction, we show that the multiscale structure can be caused by the non-convexity of the models and the non-uniformity of the training data. Our construction is simple with a two-layer neural network model with only a few neurons. Despite its simplicity, it can already reveal the origin of complicated loss landscapes for neural network models. It also justifies that the simplified problems studied in our theoretical analysis are strongly connected with real neural network loss functions. 

As a summary, our contributions are:
\begin{enumerate}
\item We visualize the loss landscape of neural networks in a region that cannot be approximated by the second-order Taylor polynomial at the minimum. We observe the multiscale structure of the loss functions, exhibited in two ways: subquadratic growth near minima and separate scales in larger regions. 

\item Using the subquadratic growth of the loss functions, we theoretically explain the edge of stability phenomenon of GD. 

\item By the separate scales structure of the loss functions, we provide detailed understandings for the working principle and necessity of learning rate decays during the optimization process, even for the deterministic GD algorithm.

\item We give a simple, yet neural network relevant, construction in which both the subquadratic behavior and separate scales structure happens for the loss landscape. The construction shows that such properties can be caused by the non-uniformity of the training data. 
\end{enumerate}

\paragraph{Related works.}
Many works that study the optimization behavior of neural networks resort to a local quadratic approximation of the loss function. This approach is equivalent with linearizing the optimization dynamics around the minima, or fixing the second-order derivatives. One line of works uses this idea to study the minima selection effect of optimizers by analyzing their linear stability in the quadratic approximation~\cite{wu2018sgd,giladi2019stability,ma2021linear}. 
Another notable series of works prove convergence of the GD dynamics for highly over-parameterized neural networks using the fact that the initialization is already in a region with good quadratic approximation~\cite{Jacot2018Neural,du2018gradient,allen2019convergence,chizat2019lazy,ma2019comparative,zou2020gradient}. This technique is usually referred to as lazy training~\cite{chizat2019lazy} or Neural Tangent Kernel (NTK)~\cite{Jacot2018Neural}. Besides direct analysis on (approximately) quadratic landscape, the Hessian of the loss function is widely used to study and characterize the landscape. For example, eigenvalues of the Hessian are used to measure the flatness of minima~\cite{dinh2017sharp}. It is also used to study the escaping of SGD from local minima~\cite{dai2018towards}.

The mechanism of SGD's exploration among different minima is made clear in the recent work~\cite{li2021happens}, which characterizes the movement of SGD iterators along the minima manifold. This picture of exploration along minima manifolds suits the neural network problem better than the exploration among isolated minima. Prior to~\cite{li2021happens}, similar analysis has been conduct for SGD with label noise~\cite{damian2021label,blanc2020implicit}. In this paper, we also analyze the motion of optimizers along the minima manifold (see Section~\ref{ssec:manifold}). Our theory is essentially different from these works, because we consider deterministic GD rather than SGD, and the motion along the manifold in our case is driven by an interaction of subquadratic growth and changing flatness, rather than the SGD noise.

The optimization on mulitscale objective function has also been studied. For example, in~\cite{kong2020stochasticity}, a diffusion effect was derived from deterministic gradient descent due to the small scales of objective functions. In the very recent work~\cite{ahn2022understanding}, the authors generalize the edge of stability phenomenon into the concept ``unstable convergence'', which happens when the objective function is complicated. Some examples studies therein is similar to the ones we study in this paper. 


\paragraph{Organization of the paper.}
The rest of the paper is organized as follows: In Section~\ref{sec:obs}, we discuss two empirical observations that show rich behaviors of GD/SGD on neural network loss functions, that cannot be easily explained by quadratic approximation. Then, in Section~\ref{sec:exp_subquadratic} we visualize the loss landscape and summarize two aspects of the multiscale structures of the loss---subquadratic growth and separate scales---that may help explain the phenomena in Section~\ref{sec:obs}. We theoretically explain the edge of stability phenomenon using the subquadratic growth in Section~\ref{sec:flattening}, and discuss how separate scales structure can help understand the behavior of learning rate decay in Section~\ref{sec:multiscale}. In Section~\ref{sec:construction}, we study the origin of the multiscale structure and construct simple examples showing the important role played by non-uniform training data. Finally, a summary and conclusions are given in Section~\ref{sec:summary}.


\section{Two empirical observations}\label{sec:obs}
In this section, we discuss two empirical observations that cannot be well explained by analyzing the optimizer on a quadratic approximation of the loss function.

\subsection{The edge of stability phenomenon}
The edge of stability (EoS) phenomenon is discussed in detail in~\cite{cohen2021gradient}. It is also observed in~\cite{wu2018sgd}. The EoS states that when GD is used to train neural networks, the sharpness at the iterator (measured by the largest eigenvalue of the Hessian) tends to increase until it arrives at $2/\eta$, where $\eta$ is the learning rate. Note that $2/\eta$ is the largest sharpness that GD can be stable assuming a quadratic loss landscape. Surprisingly, even after the sharpness stabilizes, the training loss keeps decreasing. We show one example in Figure~\ref{fig:eos} (left). Extensive experiments are done in~\cite{cohen2021gradient}. 

The EoS cannot be explained on quadratic loss functions. On a quadratic loss, GD either converges or blows up exponentially fast, unless the learning rate is exactly $2/\lambda$, where $\lambda$ is the largest eigenvalue of the Hessian. In later sections, we observe that the EoS is caused by loss landscapes that grow slower than quadratic functions around the minimum, which we call {\it subquadratic growth}. We then theoretically study how EoS happens, and why the loss value keeps decreasing after EoS, on simplified landscapes.

\begin{figure}
\centering
\includegraphics[width=0.48\textwidth, height=0.3\textwidth]{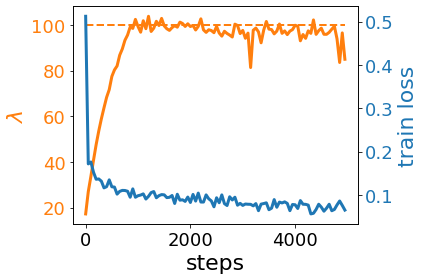}
 \includegraphics[width=0.45\textwidth, height=0.3\textwidth]{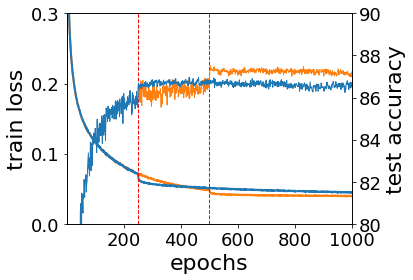}
\caption{{\bf (left)} One example of the edge of stability. We train a three-layer fully-connected neural network on a subset of CIFAR10. The orange curve shows that the sharpness of the landscape at the iterator first increases, and then stabilizes around $2/\eta$ (the dashed orange line), where $\eta$ is the learning rate. The blue curve shows the training loss keeps decreasing after the edge of stability is achieved. {\bf (right)} Training loss and test accuracy curves for two experiments with learning rate decay at different times. The blue curves show an experiment with LRD at epoch 250 (the left red vertical line), and the orange curves show an experiment with LRD at epoch 500 (the right red vertical line). The two experiments have the same initialization. The learning rate is $0.1$ initially and dropped to $0.01$.}
    \label{fig:eos}
\end{figure}

\subsection{The effect of learning rate decay}
Another observation that cannot be explained by quadratic approximation is the effect of learning
rate decay (LRD). LRD not only helps find parameters
with lower training loss, but also benefits generalization if used at a proper time. In
Figure~\ref{fig:eos} (right), we show that doing LRD later gives better generalization performance,
even though in both experiments the learning rate is decayed after the test accuracy is nearly stable and
increasing very slowly. The figure also shows that the training loss decreases very slowly
after decay. 

The explanation of the phenomena shown in LRD is beyond the reach of quadratic approximation and relies on more complicated structures of the loss. For quadratic loss function, LRD makes convergence slower, but the iterators will finally converge to a same solution and show the same generalization performance. In this paper, we will explain these observations using the separate scales structure of the loss functions.

\section{Loss landscape around training trajectory}\label{sec:exp_subquadratic}
In this section, we visualize the landscape of neural network's loss functions around the training trajectory. We observe the subquadratic growth and separate scales phenomena of the loss functions. Both these characteristics are aspects of the multiscale structure of the loss function---one with a continuum of scales and one with finite scales.

It is important to note that neural network's loss landscape possesses very rich behaviors, and almost any curvature can be found somewhere in the parameter space~\cite{skorokhodov2019loss}. In this work, we are only concerned with the loss curvature around the trajectory of SGD or GD. It is widely known that these optimizers only explore a very confined but important region of the whole parameter space. 

\subsection{The subquadratic growth}
In the experiments shown in Figure~\ref{fig:flattening}, we train neural networks using GD until the loss stops decreasing, or decreases very slowly, in which case we suppose GD starts oscillating around some minima. Then, we pick a point on the GD trajectory and visualize the ``gradient direction loss landscape'' around this point---the loss landscape along the line going through this point and towards the gradient direction at this point. The gradient direction landscape is important because it is the landscape that GD sees. Experiments are conducted on VGG network, ResNet, and DenseNet, on CIFAR10 and CIFAR100 datasets. The results for gradient direction landscapes are shown in the first row of Figure~\ref{fig:flattening}. The figures show that around the minimum the gradient direction landscape is convex and grows slower than quadratic functions. This subquadratic growth is verified by the second-order finite differences shown in the second row of Figure~\ref{fig:flattening}.

\begin{figure}[h!]
    \centering
    \includegraphics[width=0.32\textwidth,height=0.25\textwidth]{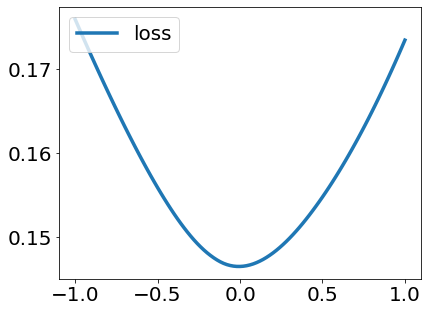}
    \includegraphics[width=0.32\textwidth,height=0.25\textwidth]{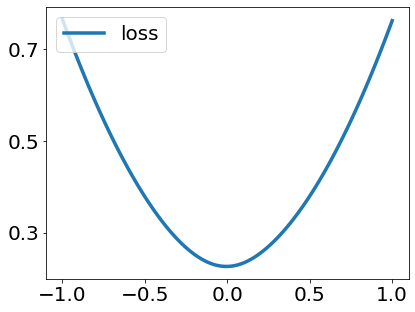}
    \includegraphics[width=0.32\textwidth,height=0.25\textwidth]{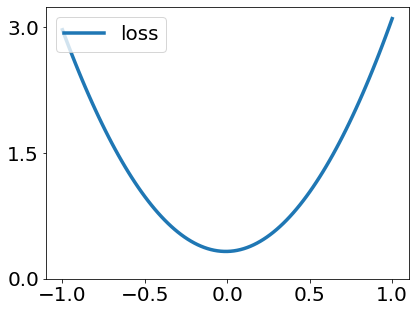} \\
    \includegraphics[width=0.32\textwidth,height=0.25\textwidth]{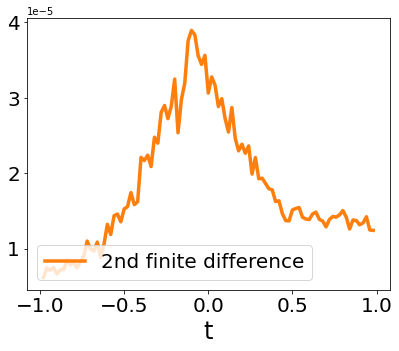}
    \includegraphics[width=0.32\textwidth,height=0.25\textwidth]{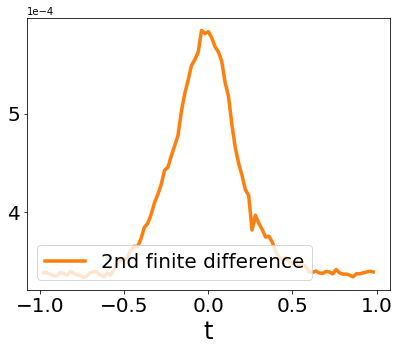}
    \includegraphics[width=0.32\textwidth,height=0.25\textwidth]{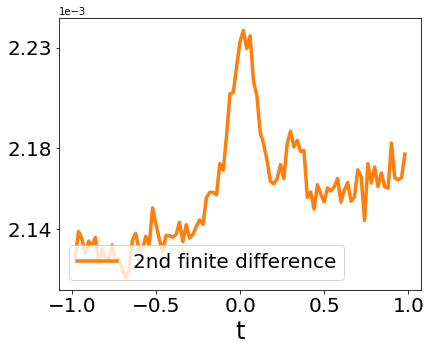} \\
    \caption{The loss landscape on gradient directions for GD when the iterator oscillates around the (local or global) minimum. The first row shows the loss, and the second row shows the second-order finite difference of the loss. Models and datasets: {\bf(left)} VGG11 on CIFAR10; {\bf(middle)} ResNet18 on CIFAR10; {\bf(right)} DenseNet121 on CIFAR100, only 50 classes are taken with 100 images per class. Note that on the second row we compute the second order finite difference of vectors (of loss values we plot on the first row). It is proportional but not equivalent to the second-order derivative.}
    \label{fig:flattening}
\end{figure}

This subquadratic growth of the landscape around minimum (at least along the gradient direction) explains the edge of stability phenomenon. The stable learning rate for GD increases as the parameters move close to the minimum. Hence, when the learning rate is not small enough, the iterator becomes unstable when it is too close to the minimum, and hence can only oscillate around the minimum at a certain level---but it may not blow up. Note that this subquadratic growth is not contradictory with the local quadratic approximation---the Taylor expansion of the loss function still holds locally, but the radius of this region is very small. 

\begin{figure}[h!]
    \centering
    \includegraphics[width=0.32\textwidth,height=0.25\textwidth]{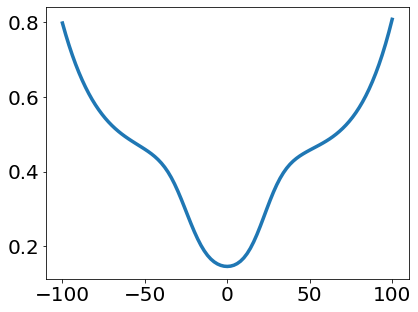}
    \includegraphics[width=0.32\textwidth,height=0.25\textwidth]{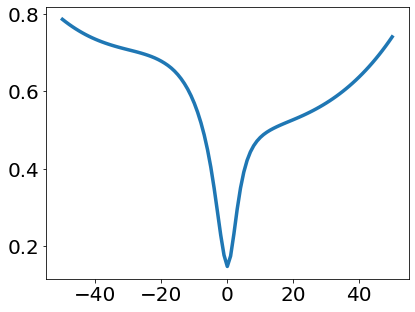}
    \includegraphics[width=0.32\textwidth,height=0.25\textwidth]{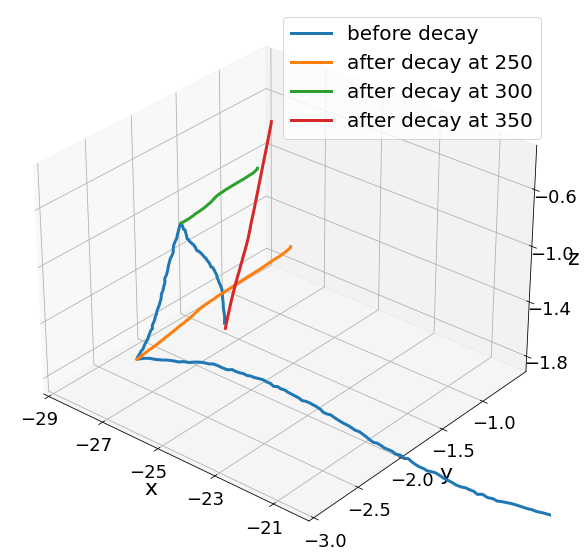}
    \caption{The landscape around a point on the GD trajectory of training a VGG11 network on CIFAR10. The left panel takes the gradient direction, and the middle panel takes a random direction. The range of visualization is larger than that in Figure~\ref{fig:flattening}. We can see separate scales in the loss landscape. The right panel shows the parameter vector of a VGG11 model projected in a 3-D space. The results of $3$ experiments are shown, with learning rate decayed at epoch $250$, $300$, and $350$, respectively. The x, y, z directions are obtained by orthogonalizing the first leading principal component of the blue segment before LRD (parameters with large learning rate), the orange segment (parameters after LRD at epoch 250), and the red segment (parameters after LRD at epoch 350), respectively.}
    \label{fig:multiscale1}
\end{figure}

\subsection{The separate scales structure}
If the landscape is visualized in a larger domain, we will inevitably see non-convex behaviors of the loss functions. Some examples are shown in the left and middle panels of Figure~\ref{fig:multiscale1}.
Here, we can observe another typical structure of neural network's loss functions---separate scales. Figures in~\ref{fig:multiscale1} show that the minimum lies in a small, sharp well located within a large, flat well. The separate scales structure can give richer behaviors in high dimensional spaces. For example, in the right panel of Figure~\ref{fig:multiscale1} we indirectly show the multiscale loss by visualizing the principal moving directions of the model parameters before and after learning rate decay. In the figure the iterator goes along very different directions when the learning rate is decayed at different epochs, showing rich fine-scale structures of the loss function, which are different from the large-scale structure reflected by the blue curves picked by large learning rate. Later in Section~\ref{sec:multiscale} we show this high dimensional multiscale structure is an important cause of LRD's complicated behaviors.

\section{The subquadratic property and the edge of stability}\label{sec:flattening}
In Section~\ref{sec:exp_subquadratic}, we observed a subquadratic growth property of the loss landscape.  This subquadratic growth makes it possible
to explain the edge of stability phenomenon discussed in Section~\ref{sec:obs}. In this section, using a simple problem inspired by the
landscape curves shown in Figure~\ref{fig:flattening} we reveal the mechanism of edge of
stability---the iterator oscillates at a certain level related to the learning rate when the
learning rate is too large for the optimizer to converge.

\subsection{A one-dimensional analysis}
Consider a 1-D strongly convex objective function $f(x)$ with a global minimum at $x=0$. Without loss of generality, assume $f(0)=0$. Inspired by the observations in Figure~\ref{fig:flattening}, we assume $f$ has continuous second order derivatives, and $f''(x)$ decreases as $|x|$ increases. Hence, the function shows subquadratic growth and becomes flatter (measured by the second order derivative) as $x$ is farther from the minimum. For the convenience of analysis, we also assume $f$ is an even function, i.e. $f(x)=f(-x)$ for any $x\in\RR$. We study a GD minimizing $f(x)$ using learning rate $\eta$, whose iteration scheme is
\begin{equation}\label{eqn:gd}
    x_{t+1} = x_t - \eta f'(x_t).
\end{equation}

By classical theories of gradient descent, it is easy to show that the iteration~\eqref{eqn:gd} converges to the minimum $x=0$ if $\eta<\frac{2}{f''(0)}$. If $\eta>\frac{2}{f''(0)}$, instead, $x=0$ becomes an unstable stationary point for the dynamics. In this case, other than $x=0$, there is a 2-periodic solution for the GD: the iterator jumps between $x_\eta$ and $-x_\eta$, where $x_\eta$ satisfies $\eta f'(x_\eta)=2x_\eta$. We assume $x_\eta>0$ and denote this periodic solution by $\{\pm x_{\eta}\}$. The following simple lemma shows that $x_\eta$ exists for any $\eta>\frac{2}{f''(0)}$ as long as $f''(x)$ goes to $0$ as $|x|$ tends to infinity, and $x_{\eta}$ is increasing with respect to $\eta$. 

\begin{lemma}\label{lm:x_eta}
If $\lim\limits_{|x|\rightarrow\infty} f''(x)=0$, $x_\eta$ exists for any $\eta>\frac{2}{f''(0)}$. Moreover, viewed as a function of $\eta$, $x_\eta$ is monotonically increasing. 
\end{lemma}

\begin{proof}
Note that if $x_\eta$ exists for some $\eta>0$, we have $\eta f'(x_\eta)=2x_\eta$, which means 
\begin{equation*}
    \frac{2}{\eta} = \frac{f'(x_\eta)}{x_\eta}.
\end{equation*}
Hence, let $h(x)=f'(x)/x$, it suffices to show that $h(x)$ is a decreasing function in $(0,\infty)$, and $\lim\limits_{x\rightarrow0^+} h(x)=f''(0)$, and $\lim\limits_{x\rightarrow\infty} h(x)=0$.

We first show the monotonicity of $h$. Taking derivative of $h$, we have
\begin{equation*}
    h'(x) = \frac{f''(x)x-f'(x)}{x^2}.
\end{equation*}
Since $f''(x)$ is decreasing, we have
\begin{equation*}
    f''(x)x=\int_0^x f''(x)dt \leq \int_0^x f''(t)dt = f'(x)-f'(0) = f'(x).
\end{equation*}
Therefore, $f''(x)x-f'(x)\leq0$, and hence $h'(x)\leq0$. This shows $h(x)$ is monotonically decreasing. 

For the limits, writing $f'(x)$ as integral of $f''(x)$, we have
\begin{equation*}
    h(x)=\frac{1}{x}\int_0^x f''(t)dt.
\end{equation*}
By L'hopital's rule we obtain $\lim\limits_{x\rightarrow0^+} h(x)=f''(0)$. On the other side, by $\lim\limits_{x\rightarrow\infty} f''(x)=0$ we easily have $\lim\limits_{x\rightarrow\infty} h(x)=0$.
\end{proof}

Next, we consider the GD dynamics, and show that if the objective function is ``strictly subquadratic'', i.e. $f''(x)$ is strictly decreasing as $|x|$ increases, then the GD iterator converges to the periodic solution $\{\pm x_\eta\}$ except a zero-measure set of $x_0$.

\begin{theorem}\label{thm:2period}
Assume $f$ satisfies the conditions in Lemma~\ref{lm:x_eta}, and $f''(s)<f''(t)$ for any $|s|>|t|$. Let $x_0, x_1, ...$ be the GD trajectory following the iteration scheme of ~\eqref{eqn:gd} starting from some $x_0\in\RR$ with learning rate $\eta$. Then, except a zero-measure set over the choice of $x_0$, we have $x_t$ converges to $0$ if $\eta\leq \frac{2}{f''(0)}$, otherwise $x_t$ converges to the periodic solution $\{\pm x_\eta\}$. 
\end{theorem}

\begin{proof}
When $\eta<\frac{2}{f''(0)}$, the proof of convergence is standard. When $\eta=\frac{2}{f''(0)}$, note that $f''(x)<f''(0)$ for any $x\neq0$. Then, for any $x\neq0$, we have 
\begin{equation*}
|\eta f'(x)| = \eta \left|\int_0^{x} f''(t)dt\right| < \eta\left|\int_0^{x} f''(0)dt\right| = \eta xf''(0) = 2x,
\end{equation*}
which implies
\begin{equation*}
|x-\eta f'(x)| < |x|.
\end{equation*}
This gives convergence of the GD trajectory to $0$.

Next, we consider the case when $\eta>\frac{2}{f''(0)}$. Let $A$ be the set of those $x_0$ such that starting from these $x_0$ the GD will arrive at $0$ after some steps. Because in the current case $0$ is an unstable stationary point, it is easy to show that $A$ contains countable number of points and hence has zero Lebesgue measure~\cite{wang2021large}. We ignore the detailed proof here. 

For any $x_0\in\RR\backslash A$, let $x_1, x_2, x_3, ...$ be the sequence of points generated by GD. Then, we have $x_k\neq0$ for any $k=0,1,2,...$. In this case, we show that for any $k$, the distance of $x_{k+1}$ to one of $\pm x_\eta$ is always smaller than the distance of $x_k$ to one of $\pm x_\eta$. Since $f$ is symmetric with respect to $x=0$, without loss of generality we assume $x_k>0$. Hence, $x_k$ is closer to $x_\eta$ than $-x_\eta$. We show that 
\begin{equation*}
    |x_{k+1}-(-x_\eta)| \leq |x_k-x_\eta|. 
\end{equation*}
This is equivalent with 
\begin{equation}\label{eqn:pf_conv_1}
    |x_k-\eta f'(x_k)-x_\eta + \eta f'(x_\eta)| \leq |x_k-x_\eta|.
\end{equation}
Rewriting the left hand side of~\eqref{eqn:pf_conv_1}, we have
\begin{align}
|x_k-\eta f'(x_k)-x_\eta + \eta f'(x_\eta)| &= |(x_k-x_\eta) - \eta(f'(x_k)-f'(x_\eta))| \nonumber\\
  &= \big|(x_k-x_\eta) - \eta\int_{x_\eta}^{x_k} f''(t)dt\big| \label{eqn:pf_conv_2}
\end{align}

We study the right hand side of~\eqref{eqn:pf_conv_2} in two cases:

\noindent\textbf{Case 1: $x_k>x_\eta$}

In this case, we have
\begin{equation*}
\int_{x_\eta}^{x_k} f''(t)dt < (x_k-x_\eta)f''(x_\eta) < (x_k-x_\eta)\frac{1}{x_\eta}\int_0^{x_\eta} f''(t)dt = (x_k-x_\eta)\frac{f'(x_\eta)}{x_\eta} = \frac{2}{\eta}(x_k-x_\eta).
\end{equation*}
Also considering $\int_{x_\eta}^{x_k} f''(t)dt>0$, we have
\begin{equation*}
    0 < \eta\int_{x_\eta}^{x_k} f''(t)dt < 2(x_k-x_\eta),
\end{equation*}
which implies
\begin{equation*}
\big|(x_k-x_\eta) - \eta\int_{x_\eta}^{x_k} f''(t)dt\big| < |x_k-x_\eta|. 
\end{equation*}

\noindent\textbf{Case 2: $x_k<x_\eta$}

In this case, we rewrite the right hand side of~\eqref{eqn:pf_conv_2} as
\begin{equation*}
    \big|(x_\eta-x_k) - \eta\int_{x_k}^{x_\eta} f''(t)dt\big|.
\end{equation*}
For the integral term, due to the monotonicity of $f''$, we have
\begin{equation*}
\int_{x_k}^{x_\eta} f''(t)dt < \frac{x_\eta-x_k}{x_\eta} \int_0^{x_\eta} f''(t)dt = (x_\eta-x_k)\frac{f'(x_\eta)}{x_\eta} = \frac{2}{\eta}(x_\eta-x_k).
\end{equation*}
Hence, again we have
\begin{equation*}
    0 < \eta\int_{x_k}^{x_\eta} f''(t)dt < 2(x_\eta-x_k),
\end{equation*}
and thus 
\begin{equation*}
\big|(x_\eta-x_k) - \eta\int_{x_k}^{x_\eta} f''(t)dt\big| < |x_\eta-x_k|. 
\end{equation*}

Since $x_k=x_\eta$ is the trivial case, we finish showing~\eqref{eqn:pf_conv_1}. And the only way for equality to hold is $x_k=x_\eta$. Therefore, the GD trajectory converges to $\{\pm x_\eta\}$. 
\end{proof}

Theorem~\ref{thm:2period} shows that subquadratic growth can cause the edge of stability phenomenon. When the learning rate is too big to converge, the GD does not blow up. Instead, it oscillates at a certain level related with the learning rate. During the oscillation, if the learning rate is dropped to a smaller value, the iterator will leave the current periodic solution and converge to a new periodic solution at a lower level. Results of numerical simulations are shown in Figure~\ref{fig:traj_loss_1d2d}. Similar phenomena are observed in~\cite{cohen2021gradient}.

\subsection{A multi-dimensional analysis}
The problem analyzed above possesses an accurate mathematical characterization because the objective function considered is simple. It is a 1-D function, and though it enjoys subquadratic growth, it is still convex. For more general cases, e.g. non-convex high dimensional functions, the picture is much more complicated for at least two reasons: (1) there may be more than one periodic solutions, and the period of some solutions can be very long; (2) the dynamics can easily get chaotic when the learning rate is large. See Figure~\ref{fig:traj_loss_1d2d} for some experiments. Nevertheless, as long as there is subquadratic growth of the landscape around the minimum, the GD still does not blow up for large learning rates. In this and the next subsection, we make some extensions for the theory to consider high-dimensional/non-convex functions. 

\begin{figure}[h!]
    \centering
    \includegraphics[width=0.32\textwidth,height=0.25\textwidth]{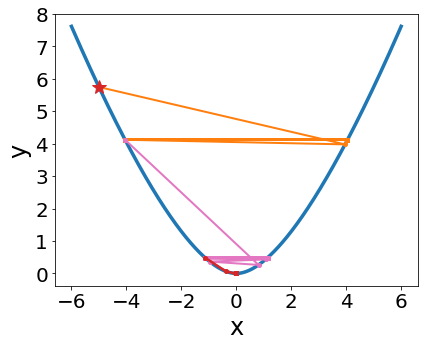}
    \includegraphics[width=0.32\textwidth,height=0.25\textwidth]{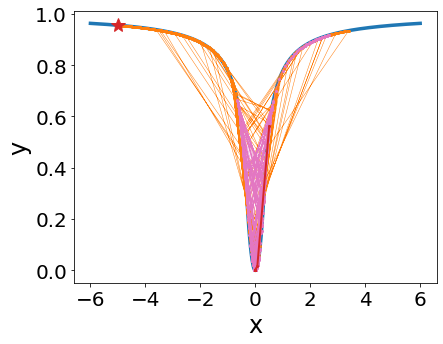}
    \includegraphics[width=0.32\textwidth,height=0.25\textwidth]{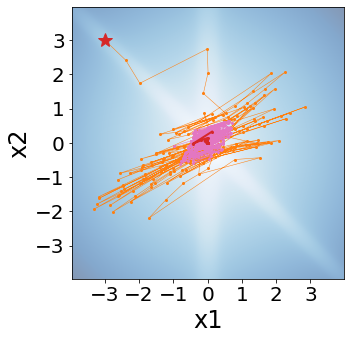} \\
    \includegraphics[width=0.32\textwidth,height=0.25\textwidth]{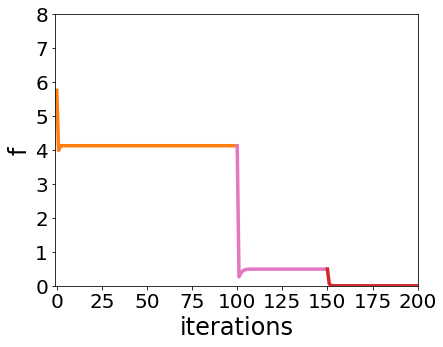}
    \includegraphics[width=0.32\textwidth,height=0.25\textwidth]{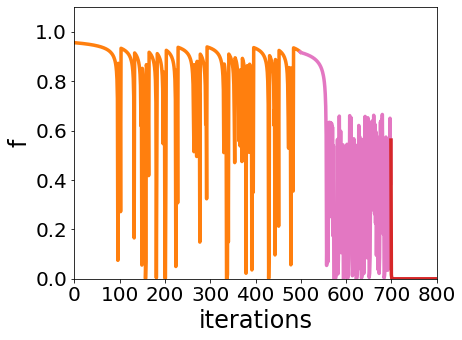}
    \includegraphics[width=0.32\textwidth,height=0.25\textwidth]{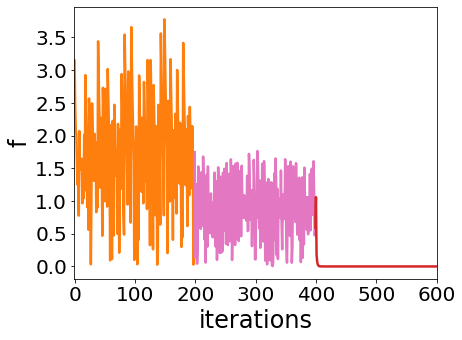} \\
    \caption{The trajectory and objective function values when minimizing subquadratic functions by GD. {\bf (left)} A 1-D convex function $f_1(x)=(1+|x|)\ln(1+|x|)-|x|$. The initialization is $x_0=-5$. The learning rate is initially $5$, and decreased to $3$ and $1$ on iteration $100$ and $150$. {\bf (middle)} A 1-D non-convex function $f_2$ given by~\eqref{eqn:fcn_integral} with $C=3$. The initialization is $x_0=-5$. The learning rate is initially $2$, and decreased to $0.5$ and $0.2$ at iterations $500$ and $700$. GD show chaotic behaviors in this case. {\bf (right)} A 2-D non-convex function $f_3(x_1, x_2)=f_2(x_1)+f_2(x_1+x_2)+0.5f_1(\sqrt{x_1^2+x_2^2})$. The initialization is $(-3,3)$. The learning rate is initially $1$, and decreased to $0.3$ and $0.1$ at iterations $200$ and $400$. GD show chaotic behaviors.}
    \label{fig:traj_loss_1d2d}
\end{figure}

We first extend our analysis to a class of high dimensional functions that can be decomposed into a sum of subquadratic functions in eigendirections. Concretely, we consider functions like
\begin{equation}\label{eqn:flatfcn_extension1}
    f(\bx) = f_1(\bp_1^T\bx) + f_2(\bp_2^T\bx) + \cdots + f_n(\bp_n^T\bx),
\end{equation}
where $\bx, \bp_i \in\RR^n$, $\{\bp_1, ..., \bp_n\}$ is an orthonormal basis of $\RR^n$, and $f_1, ..., f_n$ are subquadratic functions satisfying the conditions in Theorem~\ref{thm:2period}. This decomposition is inspired by the eigen-decomposition of quadratic functions. For such functions, applying theorem~\ref{thm:2period} we have the following results:

\begin{theorem}\label{thm:eigen_decomp}
Let $f:\RR^n\rightarrow\RR$ be a function with the form~\eqref{eqn:flatfcn_extension1}. Let $\bx_0, \bx_1, ...$ be the GD trajectory on $f$ starting from $\bx_0$ with learning rate $\eta$. For any $t\in\ZZ_+$ and $i\in\{1,2,...,n\}$, define $z_t^{(i)}=\bp_i^T\bx_t$ be the component of $\bx_t$ on $\bp_i$ direction. Then, except a zero-measure set over the choice of $\bx_0$, we have $z_t^{(i)}$ converges to $0$ if $\eta\leq\frac{2}{f''_i(0)}$, otherwise $z_t^{(i)}$ converges to a periodic solution $\{\pm x_\eta^{(i)}\}$. Here $\{\pm x_\eta^{(i)}\}$ is the periodic solution for $f_i$ with learning rate $\eta$. 
\end{theorem}

The proof of Theorem~\ref{thm:eigen_decomp} is a simple application of Theorem~\ref{thm:2period} on each of the components of $f$. This is possible because $\bp_1, ..., \bp_n$ are pairwise orthogonal. Note that Theorem~\ref{thm:eigen_decomp} does not imply that GD will converge to a unique 2-periodic solution. Actually, there are exponentially many 2-periodic solutions due to the combination of signs of each $f_i$'s periodic solution.

\subsection{A more general setting}
Next, we extend our study to more general cases, including nonconvex functions. In this subsection, we consider a wide class of subquadratic functions and show that GD with large learning rate does not diverge on these functions.

\begin{definition}\label{def:flat_fcn}
Let $f:\ \RR^n\rightarrow\RR$ be a twice continuously differentiable objective function. We call $f$ a subquadratic function if $f$ has a unique global (and local) minimum $\bx^*$, and we have $\lim\limits_{\|\bx\|\rightarrow\infty} \|\nabla f(\bx)\|/\|\bx\| = 0$, and $\nabla f(\bx)^T(\bx-\bx^*)\geq c\|\nabla f(\bx)\|\cdot\|\bx-\bx^*\|$ for any $\bx\in\RR^n$, where $c>0$ is a constant. 
\end{definition}

For quadratic functions, the gradient grows linearly with the magnitude of the input. For functions that satisfies the definition above, $\|\nabla f(\bx)\|$ grows slower than $\|\bx\|$. Thus, the function grows slower than a quadratic function.
Compared with the condition in our 1-D example, Definition~\ref{def:flat_fcn} is weaker, in the sense that $\|\nabla f(\bx)\|$ can decrease as $\|\bx\|$ gets bigger. Hence, the function $f$ can be nonconvex. On the other hand, the condition on the inner product of the $\nabla f(\bx)$ and $\bx-\bx^*$ guarantees that the gradient always has a component pointing towards $0$, which is the unique global minimum. 

For functions satisfying the definition above, we can show that GD does not diverge with any learning rate. 

\begin{theorem}\label{thm:flat_highdim}
Let $f:\ \RR^n\rightarrow\RR$ be a subquadratic function defined in Definition~\ref{def:flat_fcn}. Then, for any learning rate $\eta>0$, there exists $R_\eta>0$, such that for any GD trajectory $\bx_0, \bx_1, \bx_2, ...$ generated with learning rate $\eta$, there exists $T\in\ZZ$ that satisfies $\bx_t\in B_{R_\eta}(x^*)$ for any $t>T$. Here $B_{R_\eta}(x^*)$ denotes the closed Euclidean ball centered at $x^*$ with radius $R_\eta$. 
\end{theorem}

\begin{proof}
Without loss of generality we assume $x^*=0$. We will use $B_r$ to denote the closed ball with radius $r$ and centered at the origin. Because $\lim\limits_{\|\bx\|\rightarrow\infty} \|\nabla f(\bx)\|/\|\bx\| = 0$, we can find an $r_1$ such that for any $\|\bx\| > r_1$ we have $\|\nabla f(\bx)\| < \frac{2c}{\eta}\|\bx\|$, where $c$ is the constant in Definition~\ref{def:flat_fcn}. For such $\bx$, we can easily verify that 
\begin{equation}\label{eqn:contraction}
    \|\bx-\eta\nabla f(\bx)\| < \|\bx\|. 
\end{equation}
Hence, for any $\bx$ that satisfies $\|\bx\| > r_1$, GD sends the iterator closer to the minimum. Equivalently speaking, only when the iterator is within $B_{r_1}$ can GD send the iterator to a farther (or with equal distance) location from the minimum. Now, consider the one step GD mapping
\begin{equation*}
    h(\bx):= \bx-\eta\nabla f(\bx). 
\end{equation*}
Since $\nabla f$ is continuous, $h$ is a continuous function. Hence, due to the compactness of $B_{r_1}$, there exists $r_2>0$ such that $\|h(\bx)\|\leq r_2$ for any $\bx\in B_{r_1}$. Then, for any $\bx\in B_{r_2}$, we always have $h(\bx)\in B_{r_2}$, i.e. $B_{r_2}$ is an invariant set for GD iterations. Note that $r_2$ only depends on $c$ and $\eta$. 

We finish the proof by showing that the GD trajectory from any initialization $x_0$ will enter $B_{r_2}$. This is a natural result of~\eqref{eqn:contraction}.
\end{proof}

Theorem~\ref{thm:flat_highdim} characterizes the qualitative behavior of GD around a subquadratic minimum. With a certain learning rate, GD will oscillates in a learning rate dependent neighborhood of the minimum. Under the current conditions, we cannot fully characterize the trajectory---it may hit the global minimum in some step, or oscillates at a certain level, or oscillates chaotically in the neighborhood of the minimum. (though when $\eta$ is not too small, hitting or converging to the minimum is a zero measure event.) A typical behavior in 2-D space is shown in Figure~\ref{fig:traj_loss_1d2d}.

The analysis in this section is based on the observation of the local landscape of neural network loss functions around minima. Though for the convenience of analysis we assume the objective function has a global subquadratic behavior, this is not true for neural network loss functions. The subquadratic growth will stop when the parameter is far enough from the minimum. We will address this issue in Section~\ref{sec:multiscale} when we study the separate scales structure of loss functions.

\subsection{What happens after the edge of stability}\label{ssec:manifold}
For the objective functions we considered above, the GD iterator will oscillate around the unique minimum after arriving at the edge of stability. However, when training real neural network the iterator keeps moving and reducing the loss value even after reaching the EoS~\cite{cohen2021gradient,kunin2021limiting} (also see Figure~\ref{fig:eos}). This is mainly due to the over-parameterized nature of neural networks, which produces manifolds of minima instead of isolated minimum. Assume there is a manifold formed by global minima, taking a quasistatic approach in the direction tangent and orthogonal with the manifold, we can study how the GD iterator moves down the manifold and search for flat minima.

Let $f$ be an 1-D function that satisfies the conditions in Theorem~\ref{thm:2period}. Let $h:\RR^d\rightarrow\RR$ be a smooth positive function. Consider the function $F:\RR^{d+1}\rightarrow\RR$ defined as $F(\bx,y)=f(h(\bx)y)$, where $\bx\in\RR^d$ and $y\in\RR$. It is easy to know that the global minima of $F$ form a $d$-dimensional manifold $\{y=0\}$, and the function has a subquadratic growth in the direction orthogonal to the manifold. The value of $h(\bx)$ determines the flatness of the minimum. The smaller the $h(\bx)$, the flatter the minimum.
To consider how the GD iterator moves down the minima manifold when oscillating around it, we take a quasistatic point of view, by assuming that the $y$ component of the iterator is always bouncing between the 2-periodic solution. This assumption makes sense when $h(\bx)$ changes slowly compared with the moving speed of $\bx$, which happens when the learning rate is relatively small.

Consider a GD trajectory generated from $(\bx_0,y_0)$ using learning rate $\eta$. In the quasistatic case, the update of $\bx_t$ is
\begin{equation}\label{eqn:quasistatic1}
    \bx_{t+1} = \bx_t - \eta f'(h(\bx_t)y_t)y_t\nabla h(\bx_t),
\end{equation}
while $y_t$ follows the 2-periodic solution
\begin{equation}\label{eqn:quasistatic2}
    -y_t = y_t - \eta f'(h(\bx_t)y_t)h(\bx_t).
\end{equation}
By~\eqref{eqn:quasistatic2}, we have $\eta f'(h(\bx_t)y_t) = 2y_t/h(\bx_t)$. Substituting to~\eqref{eqn:quasistatic1} we obtain
\begin{equation}\label{eqn:quasistatic3}
    \bx_{t+1} = \bx_t - 2y_t^2\frac{\nabla h(\bx_t)}{h(\bx_t)} = \bx_t -  2y_t^2\nabla\log h(\bx_t). 
\end{equation}
Equation~\eqref{eqn:quasistatic3} shows that the motion of the GD iterator projected onto the manifold follows a GD of $\log h(\bx)$, and the speed of the dynamics is determined by $y_t$. Therefore, during the oscillation around the minima manifold, GD searches for flatter minimum by reducing the value of $h(\bx)$. An illustration of this effect is shown in Figure~\ref{fig:one_valley}.

\begin{figure}[h!]
    \centering
    \includegraphics[width=0.32\textwidth,height=0.25\textwidth]{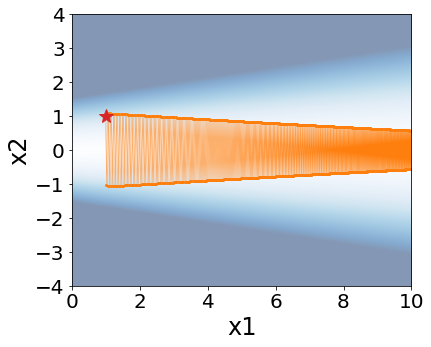}
    \includegraphics[width=0.32\textwidth,height=0.25\textwidth]{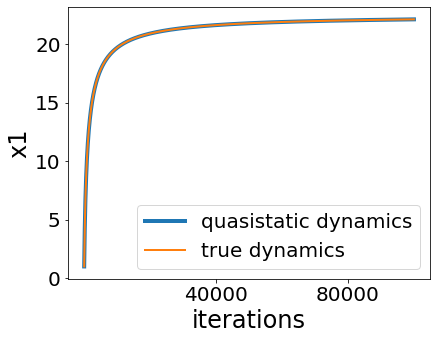}
    \includegraphics[width=0.32\textwidth,height=0.25\textwidth]{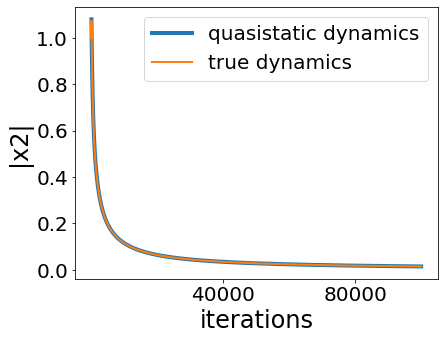}
    \caption{The trajectory of GD on a 2-D function with a flattening valley and subquadratic growth on $x_2$ direction. The iterator bounces back and forth on the valley and moves to the vicinity of flatter minima. The middle and right panels show the values of $x_1$ and $|x_2|$ for both GD and the quasistatic dynamics~\eqref{eqn:quasistatic3}. The results show: (1) the dynamics~\eqref{eqn:quasistatic3} is a very good approximation of the real dynamics. (2) The flatness-driven motion gets increasingly slower and finally stops. Experiment details: the loss function is $f(x_1,x_2)=f_1(x_2/(1+0.01x_1))$, where $f_1$ is defined in the caption of Figure~\ref{fig:traj_loss_1d2d}. The initialization is $(1,1)$. The learning rate is $3$.}
    \label{fig:one_valley}
\end{figure}

\begin{remark}
The idea of flatness driven motion along the manifold is similar to that in~\cite{li2021happens}, but our result is essentially different. We treat GD instead of SGD, and in our case, the motion along the manifold is made possible by the subquadratic landscape around the minima, instead of the SGD noise. The two types of flatness driven motion have quite different behaviors in some situation. For instance, the SGD noise drives the iterator to the flattest minimum on the manifold, while for GD it will converge after finding a sufficiently flat minimum (relative to the learning rate). Also, if the manifold consists of interpolation solutions, SGD will not show flatness driven motion because the noise vanishes at the minima. However, in our analysis movement still exists as long as the minima have subquadratic property. 
\end{remark}

\begin{remark}
The quasistatic approach also serves as a handy tool to derive the noise-driven motion along manifold for SGD. As a simple illustration, consider an objective function $f(\bx,\by)=\by^T H(\bx) \by$, which gives a quadratic approximation of a loss function with a global minima manifold $\{\by=0\}$. $H(\bx)$ always gives a positive definite matrix. Suppose an SGD is approximated by an SDE
\begin{equation*}
d\left[\begin{array}{c}\bx_t\\\by_t\end{array}\right] = - \left[\begin{array}{c}\by_t\nabla H(\bx_t)\by_t\\2H(\bx_t)\by_t\end{array}\right]dt + \left[\begin{array}{c} 0 \\ \sigma\sqrt{\eta H(\bx_t)}dW_t \end{array}\right].
\end{equation*}
Here we assume for convenience there is no noise along the manifold direction, and the noise on the $\by$ direction depends on the flatness of the minimum. Then, by assuming $\by$ is always at equilibrium, we first solve the dynamics of $\by$ fixing $\bx=\bx_t$. The equilibrium is
\begin{equation*}
    \by_\infty \sim \mathcal{N}(0, \frac{\eta\sigma^2}{4}I).
\end{equation*}
Plugging the equilibrium above into the dynamics of $\bx$, and taking expectation over $\by$, we obtain the expected dynamics of $\bx$ in quasistatic case:
\begin{equation*}
    \dot{\bx}_t = -\EE \by_\infty^T \nabla H(\bx_t) \by_\infty = -\frac{\eta\sigma^2}{4} \nabla\tr(H(\bx_t)),
\end{equation*}
which recovers the results in~\cite{li2021happens}.
This quasistatic approach can be easily adapted to other types of noise and other optimizers such as SGD with momentum.
\end{remark}

\section{The separate scales and learning rate decay}\label{sec:multiscale}

Learning rate decay is a widely adopted technique in training large scale neural networks, and has received much theoretical attention, too. Explanations of how LRD works include GD stability in different directions~\cite{lecun1990second}, SGD exploration~\cite{kleinberg2018alternative}, and pattern complexity~\cite{you2019does}. However, there are still some behaviors shown by training with LRD that cannot be well addressed by these explanations. For example, as shown in Figure~\ref{fig:eos}, the generalization performance suffers if the learning rate is decayed too early. In this section, we build a simple loss function, inspired by the observations of separate scales structure, that can explain this behavior of learning rate decay.

We will build a landscape with two valleys in different scales. To start with, consider a loss with a single valley in $\RR^2$, given by
\begin{equation*}
    g_1(x, y) := f(h(x)y).
\end{equation*}
Here, $f:\RR\rightarrow\RR$ is a function with subsquadratic growth a global minimum at $0$, $h:\RR\rightarrow\RR$ is a positive function controlling the flatness of the minimum. We assume $\argmin\{h(x)\}=0$. Hence, $(0,0)$ is the flattest minimum of $g$, among all minima with the form $(x,0)$. Next, we build another valley by scaling the $g_1$ above to a smaller scale, and rotating the valley to go through the $y$ direction. We obtain
\begin{equation*}
    g_2(x,y) = f(kh(y)x),
\end{equation*}
with $k>1$. Finally, we build a multiscale landscape with two valleys by considering
\begin{equation}\label{eqn:multiscaleF}
    F(x,y) = g_1(x,y) + \phi_c(g_2(x,y)),
\end{equation}
where $\phi_c$ is a non-decreasing cutoff function that confines the value of $g_2$ within $[0,c]$. This makes the effect of $g_2$ local. The simplest choice of $\phi_c$ can be $\phi_c(z)=\min\{z,c\}$. An example landscape of $F$ is shown in Figure~\ref{fig:two_valleys}. The landscape has one large and flat valley and one small and sharp valley. The global minimum is at the origin, locating at the bottom of both valleys.

Now we can study the behavior of GD with learning rate decay in the landscape of $F$. Since the valley of $g_2$ is sharp, a small learning rate is necessary for GD to converge. On the other hand, when initialized far away from the global minimum, the iterator will first be attracted by the large valley. GD will first converge to the neighborhood of the large valley, then move along the valley to its flattest region (near the global minimum in this case) while bouncing between the valley walls via the mechanism discussed in Section~\ref{ssec:manifold}. This moving process is faster with a larger learning rate. Therefore, an ideal strategy for learning rate decay is to use a large learning rate until the iterator moves to the vicinity of the small valley and then drop to a small learning rate to converge into the small valley. In this process, when the iterator is bouncing around and moving down the large valley, the loss value is not decreasing much. But that does not mean the learning rate can be decayed earlier. If it is decayed before the iterator is close enough to the small valley, it then has to move down the large valley using the small learning rate, which can cost much more time. In the extreme case, a small learning rate can cause convergence to a suboptimal minimum on the large valley. 

An numerical example with the form~\eqref{eqn:multiscaleF} is given in Figure~\ref{fig:two_valleys}. In the experiment, we initialize GD at a point far from the small valley with a large learning rate. Afterwards, the learning rate is decayed by a same factor at different steps. Although the loss values when the learning rate is decayed are similar, the three trajectories take drastically different amount of time to converge. 

\begin{figure}[h!]
    \centering
    \includegraphics[width=0.32\textwidth,height=0.25\textwidth]{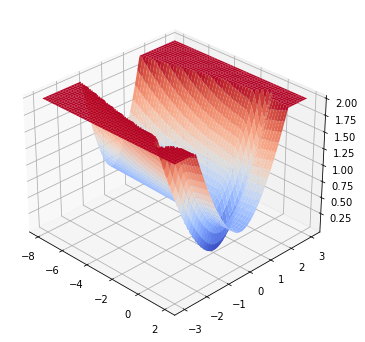}
    \includegraphics[width=0.32\textwidth,height=0.25\textwidth]{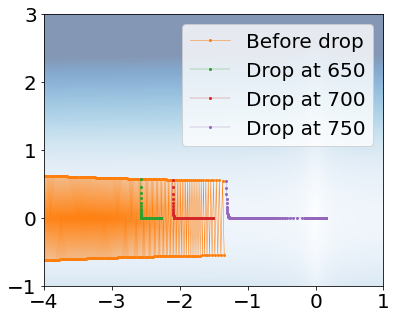}
    \includegraphics[width=0.32\textwidth,height=0.25\textwidth]{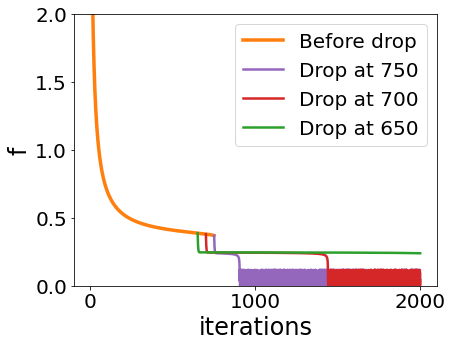}
    \caption{Learning rate decays at different steps lead to very different behavior after decaying. GD is initialized at $\bx_0=(-20,3)$ with learning rate $2.5$, and then decayed to $0.25$ at step $650$, $700$ or $750$. The loss function is $f(x_1, x_2)=h(x_1,x_2)+0.25h(x_2,10x_1)/(1+h(x_2,10x_1))$, with $h(x_1,x_2)=f_1((1+0.001x_1^2)x_2)$, where $f_1$ is defined in the caption of Figure~\ref{fig:traj_loss_1d2d}. {\bf (left)} The 2-D landscape of a loss function with two valleys in two scales. {\bf (middle)} The GD trajectories around and after learning rate decays. The ``before drop'' curve shows the GD trajectory from iteration 466 (when the $x$ coordinate gets bigger than $-4$) to iteration 750. The other three curves show the GD trajectores after learning rate decays. Each of them shows 500 iterations after the decay. {\bf (right)} The loss curves.}
    \label{fig:two_valleys}
\end{figure}

\section{The origin of the multiscale structure from training data}\label{sec:construction}

Both the subquadratic growth and separate scale loss structures can be understood as manifestations of multiscale structure---finite significant scales for separate scale loss and a continuum of scales for subquadratic growth. In this section, we study the origin of the loss's multiscale structure. By a simple neural network based construction, we show that the non-convexity of the model and the multiscale structure of the training data together act as one cause of the multiscale loss. 

Now we describe our construction. Consider the following two-layer neural network with $3$ neurons,
\begin{equation}\label{eqn:model1}
    f(x,w) = 1 -\sigma(wx+1) + 2\sigma(wx) -\sigma(wx-1),
\end{equation}
where $x$ is the input and $w$ is the parameter. Let $\sigma$ be the ReLU function. In this problem we assume the three neurons share the parameter $w$, so there is only one parameter. Suppose we have data $\{(x_i, y_i)\}_{i=1}^n$, with $x_i>0$ and $y_i=0$ for all $i$. The loss function is
\begin{equation*}
    L(w) = \frac{1}{n}\sum\limits_{i=1}^n f(x_i,w)^2.
\end{equation*}
It is easy to show that, for fixed $x$, $f(x,w)$ is the following piecewise linear function for $w$:
\begin{equation*}
f(x,w) = \left\{ \begin{array}{ll}
    1 & \textrm{if}\ w\leq-\frac{1}{x} \\
    |xw| & \textrm{if}\ -\frac{1}{x}\leq w\leq\frac{1}{x} \\
    1 & \textrm{if}\ w>\frac{1}{x}.
\end{array} \right.
\end{equation*}
Consequently, each term in the loss function is
\begin{equation*}
f(x_i,w)^2 = \left\{ \begin{array}{ll}
    1 & \textrm{if}\ w\leq-\frac{1}{x_i} \\
    x_i^2w^2 & \textrm{if}\ -\frac{1}{x_i}\leq w\leq\frac{1}{x_i} \\
    1 & \textrm{if}\ w>\frac{1}{x_i}.
\end{array} \right.
\end{equation*}
This is a function which is quadratic around $0$, and takes constant when $w$ is away from $0$. Moreover, the width of the quadratic segment depends on the magnitude of $x$. For bigger $x$, the quadratic part is narrow while sharp. For smaller $x$, it is wide and flat. From these properties, it is easy to show that the total loss $L$ gets sharper for $w$ closer to $0$. When $x_i$'s vary a lot in their magnitudes, $L$ will have multi-scale structure---the sharpness increases by orders of magnitudes as $w$ moves towards 0. 

To be clearer, assume without loss of generality that $x_1\leq x_2\leq\cdots\leq x_n$, then $L(w)$ is the following piecewise quadratic/constant function:
\begin{equation*}
L(w) = \left\{ \begin{array}{ll}
     \left(\frac{1}{n}\sum\limits_{i=1}^n x_i^2\right)w^2 & \textrm{if}\ |w|\leq\frac{1}{x_n} \\
     \left(\frac{1}{n}\sum\limits_{i=1}^k x_i^2\right)w^2 + \frac{n-k}{n} & \textrm{if}\ \frac{1}{x_{k+1}}\leq|w|\leq\frac{1}{x_{k}},\ \textrm{for}\ k=1,2,...,n-1 \\
     1 & \textrm{if}\ |w|>\frac{1}{x_1}.
\end{array} \right.
\end{equation*}
If we have a continuum of $x$, sampled from a probability distribution $\mu$ supported on $(0,\infty)$, we can also write down the ``population'' loss function:
\begin{equation*}
L(w) = \int f(x,w)^2 d\mu(x) = \left(\int_0^{1/|w|}x^2d\mu(x)\right)w^2 + \int_{1/|w|}^\infty d\mu(x).
\end{equation*}

If the training data contains $3$ data points with different magnitude, the empirical loss function looks like the blue curve in Figure~\ref{fig:ms_construction} (left). We obtain a loss with separate scales. 
On the other hand, if $\mu$ is a uniform distribution on $[0,C]$, then the population loss is 
\begin{equation}\label{eqn:fcn_integral}
    L(w) = \left\{\begin{array}{ll}
        \frac{C^2}{3}w^2 & \textrm{if}\ |w|\leq\frac{1}{C} \\
        1-\frac{2}{3C|w|} & \textrm{if}\ |w|>\frac{1}{C}.
    \end{array} \right.
\end{equation}
The curve is shown in orange in Figure~\ref{fig:ms_construction} (left). Now, we obtain a loss with continuous scales and thus shows subquadratic growth. Similar to the population loss, if the training data set is very large, the empirical loss function will also have (nearly) continuous scales. 

\begin{figure}[h!]
    \centering
    \includegraphics[width=0.38\textwidth,height=0.3\textwidth]{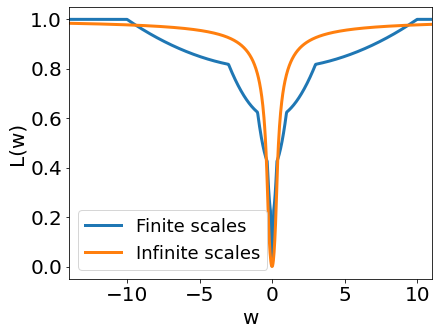}
    \includegraphics[width=0.42\textwidth,height=0.3\textwidth]{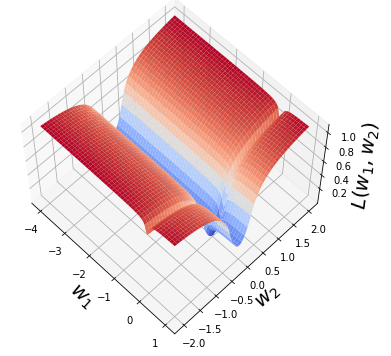}
    \caption{{\bf (left)} Example loss landscapes given by constructions in Section~\ref{sec:construction}. The curves of functions of the parameter $w$. {\bf (right)} A 2-D loss function with two valleys in different scales constructed by the 2 dimensional problem in Section~\ref{ssec:ms_construction}. The surface shows a function of the parameters $(w_1, w_2)$.}
    \label{fig:ms_construction}
\end{figure}

\paragraph{Multi-dimensional examples.}\label{ssec:ms_construction}
With similar approach, we can also construct multiscale losses in multi-dimensional spaces, especially multiscale valleys we studied in Section~\ref{sec:multiscale}. 
Recall the definition of $f(x,w)$ above. To produce a 2-D loss landscape with two valleys, we just need to consider a model with a two-dimensional input $(x^{(1)},x^{(2)})$, two parameters $w_1,w_2$, and a two-dimensional output:
\begin{equation*}
   [f(x^{(1)},w_1),\ \ f(x^{(2)},w_2)]^T.
\end{equation*}
Then, consider data $\{(x^{(1)}_i,x^{(2)}_i, y^{(1)}_i, y^{(2)}_i)\}_{i=1}^n$ with $y^{(1)}_i = y^{(2)}_i =0$ for all $i$. The loss function becomes
\begin{equation*}
    L(w_1,w_2) = \frac{1}{n}\sum\limits_{i=1}^n f(x_i^{(1)}, w_1)^2 + f(x_i^{(2)}, w_2)^2.
\end{equation*}
The loss function is a superposition of two valleys along $w_1$ and $w_2$ directions. If we assume $x^{(2)}$ is bigger than $x^{(1)}$ (e.g. $x^{(2)_i}=10x^{(1)}_i$ for any $i$), then the valley along $w_1$ direction (generated by landscape of $w_2$) is in a smaller scale than the other valley. An example using a continuous distribution of training data is given in Figure~\ref{fig:ms_construction}.

\paragraph{Discussion on homogeneity.}
Many factors can contribute to the special structure of neural network's loss functions. In the construction in this section, we focus on the distribution of training data. We show that if the training data are not well normalized, multiscale structure will appear in the loss landscape. Besides the training data, nonconvexity of the model and the loss also plays an important role in this example. If the loss for each data is convex or even quadratic, then the total loss as a sum of many single losses will not show very rich structures. 

Although in practice the input data is always standardized before being fed into the network, we note that usual standardization applies a fixed transform to all data to achieve zero mean and identity variance, but does not eliminate the multiscale structure in the data. After standardization, the length of different data can still differ by several orders of magnitude. Hence, it is still possible for training data to cause multiscale structure in the loss function. 

Finally, we show that for large deep neural networks, multiscale data can still cause multiscale
loss due to the homogeneity of ReLU function. For example, consider an L-layer fully connected
neural networks with ReLU activation function and without bias:
\begin{equation*}
    f(\bx; W_1, W_2, ..., W_L) = W_L\sigma(W_{L-1}\sigma(\cdots W_2\sigma(W_1\bx)\cdots)).
\end{equation*}
Let $l_1$ and $l_2$ be the losses of two input data $\bx_1$ and $\bx_2$, both with the same target $y$. Then, if $\bx_2=k\bx_1$, we can easily verify
\begin{equation*}
    l_1(W_1, ..., W_L) = l_2(k^{-1/L}W_1, ..., k^{-1/L}W_L),
\end{equation*}
i.e. the two losses have the same shape but different scales. Such relation is not unique. For instance, fixing $W_2, ..., W_L$, we have
\begin{equation*}
    l_1(W_1, ..., W_L) = l_2(k^{-1}W_1, W_2, ..., W_L). 
\end{equation*}
In this case, $l_2$ is a scaling of $l_1$ only in the $W_1$ space. The second relation above is true even for networks with bias.

\section{Summary}\label{sec:summary}

In this paper, we study the limitations of using the quadratic approximation for neural network's loss functions and highlight the importance of a multiscale structure. Firstly, we empirically observe two manifestations of the multiscale structure---the subquadratic growth and the separate scales structure. These properties can explain some intriguing phenomena observed during the training process of neural networks. Specifically, we explain (1) the edge of stability phenomenon of GD using the subquadratic growth and (2) the behavior and effect of learning rate decay using the separate scales structure. Then, we study the origin of the multiscale structures, and show by constructive examples that non-convex models and non-uniform training data can lead to multiscale loss. It is worth noting that our study puts more focus on GD due to its simplicity. Extending the study to SGD is an important and meaningful direction for future work.

\clearpage
\bibliography{ref}
\bibliographystyle{plain}

\end{document}